\documentclass{article}
\pdfoutput=1
\usepackage{microtype}
\usepackage{subfigure}
\usepackage{booktabs} 

\usepackage{enumitem}
\usepackage{graphicx}
\usepackage{doi}

\usepackage[accepted]{icml2024}

\usepackage{amsmath}
\usepackage{amssymb}
\usepackage{mathtools}
\usepackage{amsthm}
\usepackage{xr}
\newcommand*{\addFileDependency}[1]{
}
\makeatother\newcommand*{\myexternaldocument}[1]{
}

\usepackage[capitalize,noabbrev]{cleveref}
\usepackage{xspace}
\makeatletter
\DeclareRobustCommand\onedot{\futurelet\@let@token\@onedot}
\def\@onedot{\ifx\@let@token.\else.\null\fi\xspace}

\def\eg{\emph{e.g}\onedot} 
\def\ie{\emph{i.e}\onedot} 
 
\def\etc{\emph{etc}\onedot} 
\def\wrt{w.r.t\onedot} 

\makeatother

\theoremstyle{plain}
\newtheorem{theorem}{Theorem}[section]
\newtheorem{proposition}[theorem]{Proposition}
\newtheorem{lemma}[theorem]{Lemma}
\newtheorem{corollary}[theorem]{Corollary}
\theoremstyle{definition}

\theoremstyle{remark}

\usepackage[textsize=tiny]{todonotes}

\begin{document}

\twocolumn[
\icmltitle{Deep Extrinsic Manifold Representation for Vision Tasks}



\icmlsetsymbol{equal}{*}

\begin{icmlauthorlist}
\icmlauthor{Tongtong Zhang}{sjtu}
\icmlauthor{Xian Wei}{ecnu}
\icmlauthor{Yuanxiang Li}{sjtu}

\end{icmlauthorlist}

\icmlaffiliation{sjtu}{School of Aeronautics and Astronautics,Shanghai Jiao Tong University, China}
\icmlaffiliation{ecnu}{ecnu}


\icmlkeywords{Machine Learning, ICML}

\vskip 0.3in
]



\printAffiliationsAndNotice{\icmlEqualContribution} 

\begin{abstract}
Non-Euclidean data is frequently encountered across different fields, yet there is limited literature that addresses the fundamental challenge of training neural networks with manifold representations as outputs.
We introduce the trick named Deep Extrinsic Manifold Representation (DEMR) for visual tasks in this context. DEMR incorporates extrinsic manifold embedding into deep neural networks, which helps generate manifold representations.
The DEMR approach does not directly optimize the complex geodesic loss. 
Instead, it focuses on optimizing the computation graph within the embedded Euclidean space, allowing for adaptability to various architectural requirements.
We provide empirical evidence supporting the proposed concept on two types of manifolds, $SE(3)$ and its associated quotient manifolds. 
This evidence offers theoretical assurances regarding feasibility, asymptotic properties, and generalization capability.
The experimental results show that DEMR effectively adapts to point cloud alignment, producing outputs in $ SE(3) $, as well as in illumination subspace learning with outputs on the Grassmann manifold.
\end{abstract}

\section{Introduction}
Data in non-Euclidean geometric spaces has applications across various domains, such as motion estimation in robotics \cite{byravan2017se3nets}, shape analysis in medical imaging \cite{bermudez2018manilearning} \cite{huang2021detecting_mani} \cite{yang2022nestedgrassmann}, \etc.
Deep learning has revolutionized various fields. 
However, deep neural networks (DNN) typically generate feature vectors in Euclidean space, which may not be universally suitable for certain computer vision tasks, such as estimating probability distributions for classification or rigid motion estimation.
The classification of learning problems related to a manifold depends on the application of manifold assumptions.
The first category involves signal processing on a manifold structure, with the resulting output situated within the Euclidean space. 
For example, geometric deep-learning approaches extract features from graphs, meshes, and other structures. The encoded features are then input to decoders for tasks such as classification and regression in the Euclidean space \cite{bronstein2017geometric,cao2020geometric_review,can2021dt_science_geom,bronstein2021GDL}. 
Alternatively, they can also function as latent codes for generative models \cite{ni2021manifoldGAN}.
The second category establishes continuous mappings between data residing on the same manifold to enable regressions. 
For instance, \cite{steinke2010between_mani} addresses regression between manifolds through regularization functional, while \cite{fang2023intrinsic} performs statistical analysis over deep neural network-based mappings between manifolds.
The third category of research focuses on deep learning models that have distinct Euclidean inputs and manifold outputs. This line of research often emphasizes a specific type of manifold, such as deep rotation manifold regression \cite{zhou2019continuity,levinson2020analysis_svd,rpmg}
%


 \begin{figure*} 
	\centering 
	\subfigure[Intrinsic manifold regression via energy minimization \cite{intro2manopt}]{
		\includegraphics[width=0.3\linewidth,scale=1.00]{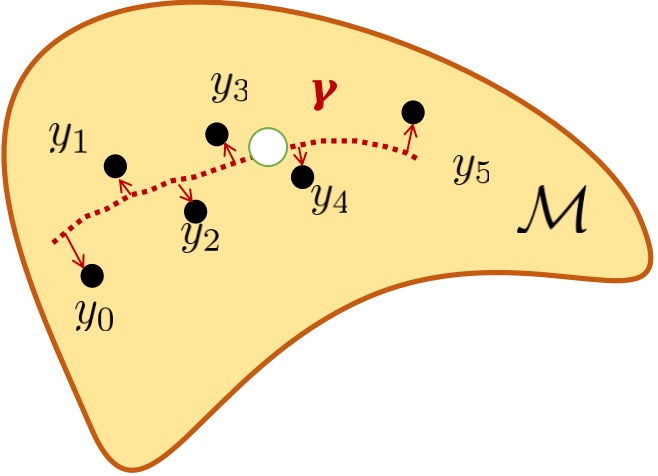} 
		\label{fig:energy_mani}}
  \hspace{0.2mm}
	\subfigure[Intrinsic manifold regression from tangent bundle \cite{zhang2020bayesian_geo}]{%
		\includegraphics[width=0.3\linewidth,scale=1.00]{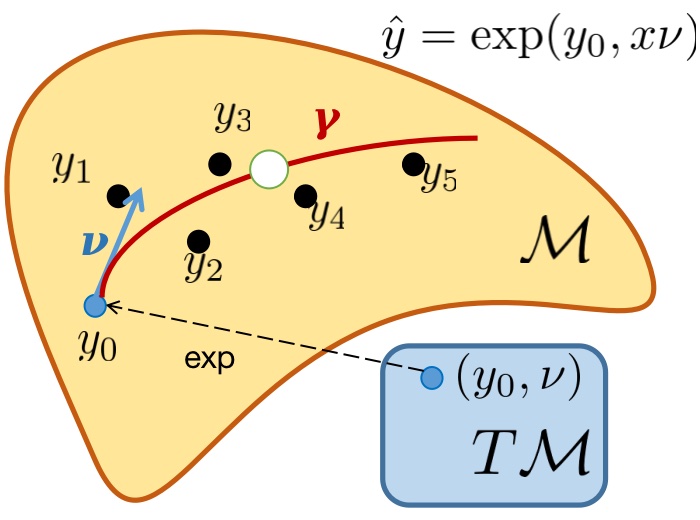} 
		\label{fig:geod_mani}}
   \hspace{0.2mm}
	\subfigure[Extrinsic manifold regression \cite{lin2017extrinsic,lee2021robustextrinsic}]{%
		\includegraphics[width=0.35\linewidth,scale=1.00]{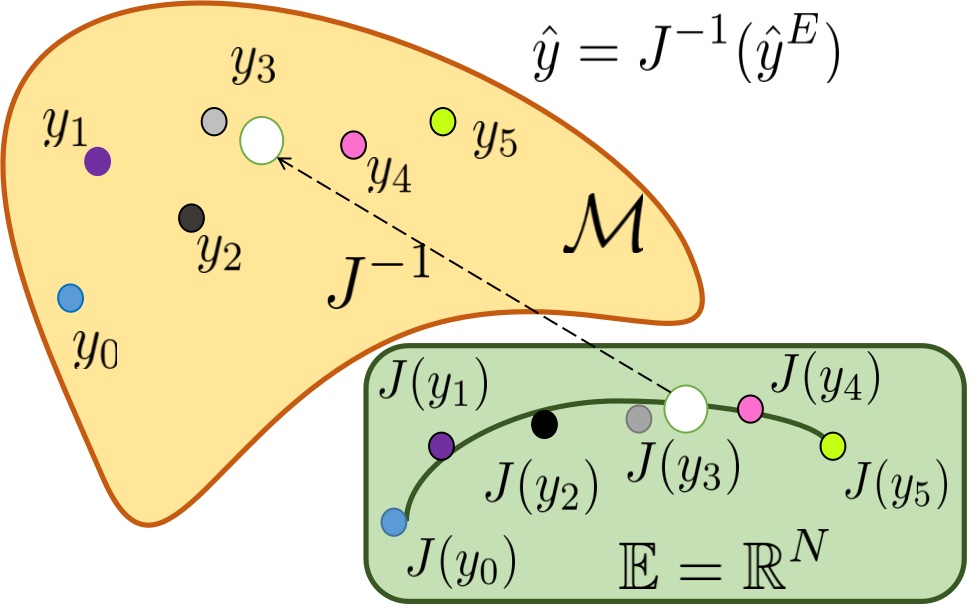} 
		\label{fig:extr_mani}}
  \vspace{-0.5cm}
		\caption{
  Manifold regression explores the relationship between a manifold-valued variable and a value in vector space. A typical intrinsic manifold regression finds the best-fitted geodesic curve $\gamma$ on $\mathcal{M}$ via (a) minimizing a complex energy function of distance and smoothness, or (b) updating parameters in the local tangent bundle $T\mathcal{M}$. 
  Extrinsic manifold regression (c) models the relationship in the extrinsically embedded space.}
   \label{fig:comparison}
   \vspace{-0.6cm}
\end{figure*}

The paper is centered on the third category, which entails creating multiple representations from DNNs. It's worth noting that models that produce outputs on the manifold are typically regularized using geometric metrics, which can be categorized into two types: intrinsic manifold loss and extrinsic manifold loss \cite{bhattacharya2003largesampletheorey,bhattacharya2012extrinsic_machine_vision}, as depicted in Figure \ref{fig:comparison}.
Intrinsic methods aim to identify the geodesic that best fits the data to preserve the geometrical structure \cite{fletcher2011geodesicregress,fletcher2013geodesic,cornea2017geodregression_on_sym,shin2022robust_geo}. 
However, the inherent characteristics of intrinsic distances pose challenges for DNN architectures.
Primarily, many intrinsic losses incorporate intricate geodesic distances, aiming to induce longer gradient flows throughout the entire computation graph \cite{fletcher2011geodesicregress,hinkle2012polynomial,fletcher2013geodesic,shi2009intrinsic,berkels2013discrete_shapespace,fletcher2011geodesicregress}.
Secondly, directly fitting a geodesic by minimizing distance and smoothness energy in the Euclidean space might result in off-manifold points \cite{rpmg,khayatkhoei2018disconnected}.

In contrast, extrinsic regression uses embeddings in a higher-dimensional Euclidean space to create a non-parametric proxy estimator. 
The estimation on the manifold $\mathcal{M}$ can be achieved using $ J^{-1} $, which represents the inverse of the extrinsic embedding $J$.
Extensive investigations in \cite{bhattacharya2012extrinsic_machine_vision,lin2017extrinsic} have established that extrinsic regression offers superior computational benefits compared to intrinsic regression.
Many regression models are customized for specific applications, utilizing exclusive information to simplify model formulations that include explicit explanatory variables. This customization is evident in applications such as shapes on shape space manifolds \cite{berkels2013discrete_shapespace,fletcher2011geodesicregress}.

However, within the computer vision community, deep neural networks are often faced with a large amount of diverse multimedia data. Traditional manifold regression models struggle to handle this varied modeling task due to limitations in representational power. Some recent works have addressed this challenge, such as \cite{fang2023intrinsic} processing manifold inputs with empirical evidence.
This paper presents the idea of embedding manifolds externally at the final regression layer of various neural networks, including ResNet and PointNet. This idea is known as \textit{ Deep Extrinsic Manifold Representation (DEMR)}.
The process is adapted from two perspectives to bridge the gap between traditional extrinsic manifolds and neural networks in computer vision.
Firstly, the conventional choice of a proxy estimator, often represented by kernel functions, is substituted with feature extractors in DNNs. Feature extractors like ResNet or PointNet, renowned for their efficacy in specific tasks, significantly elevate the representational power for feature extraction.

Secondly, to project the neural network output onto the preimage of $J(\cdot)$, we depart from deterministic projection methods employed in traditional extrinsic manifold regression. Instead, we opt for a learnable linear layer commonly found in DNN settings. This learnable projection module aligns seamlessly with most DNN architectures and eliminates the need for the manual design of the projection function $Pr(\cdot)$, a step typically required in prior extrinsic manifold regression models to match the type of the manifold.
These choices not only enhance the model's representational power compared to traditional regression models but also allow for the preservation of existing neural network architectures.

\paragraph{Contribution}
We facilitate the generation of manifold output from standard DNN architectures through extrinsic manifold embedding. In particular, we elucidate the rationale behind pose regression tasks performing more effectively as a specialized instance of DEMR. Additionally, we offer theoretical substantiation regarding the feasibility, asymptotic properties, and generalization ability of DEMR for $SE(3)$ and the Grassmann manifold. Finally, the efficacy of DEMR is validated through its application to two classic computer vision tasks: relative point cloud transformation estimation on $SE(3)$ and illumination subspace estimation on the Grassmann manifold.

\section{DEMR}
\label{sec:DEMR}
\subsection{Problem Formulation}
\paragraph{Estimation in the embedded space}
For distribution $\mathcal{Q}$ on manifold $\mathcal{M}$ of dimension $d$,  the extrinsic embedding  $\Tilde{\mathcal{M}} = J(\mathcal{M})$, from manifold $\mathcal{M}$ to Euclidean space $\mathbb{E}=\mathbb{R}^N$, has distribution $\Tilde{\mathcal{Q}} = \mathcal{Q}\circ J$, which is a closed subset of  $\mathbb{R}^N$, where $d\ll N$. 
In extrinsic manifold regression,  $\forall u\in \mathbb{E}$, $\exists$ a compact projection set  
$Pr(u) = \{x\in \Tilde{\mathcal{M}}:\|x-u\|\leq \|y-u\|,  \forall y\in \Tilde{\mathcal{M}}\}$, mapping $u$ to the closest point on $\mathcal{M}$. 
The extrinsic mean set of $\mathcal{Q}$ is 
$\mu^{ext} = J^{-1}(Pr(\mu))$, where $\mu$ is the mean set of $\Tilde{\mathcal{Q}}$. In DEMR, $\mu$ is acquired by the neural networks, and the estimation on $\mathcal{M}$ is then deterministically computed.

\begin{figure*}
\centering
\includegraphics[width=16cm]{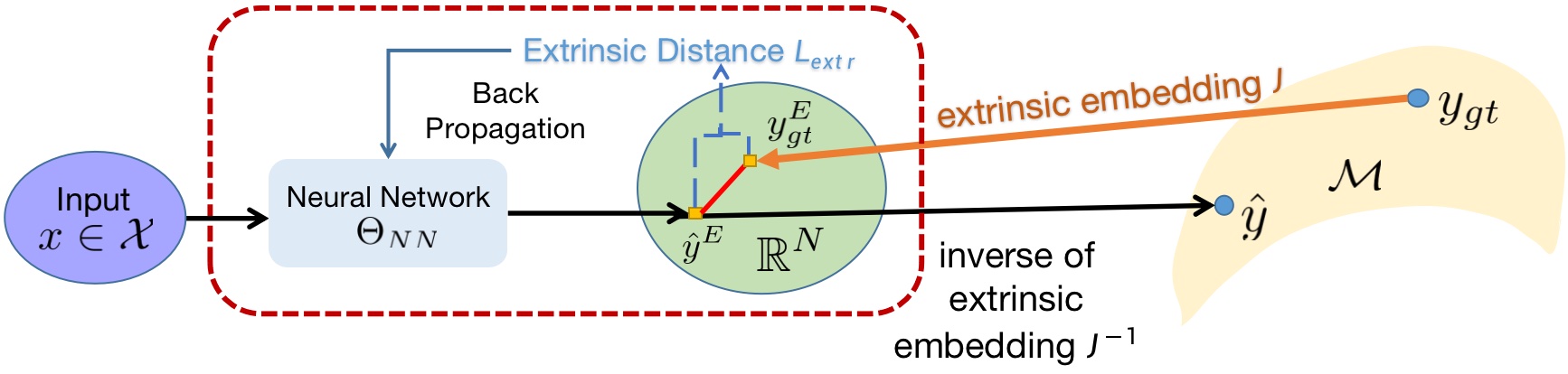}
\caption{DEMR pipeline, with black arrows indicating the forward process, and optimization in the red box.} 
 \label{fig:DEMR}
 \end{figure*}
 
\paragraph{Pipeline design}
\label{sec:pipeline}
The pipeline of DEMR is demonstrated in Figure \ref{fig:DEMR}. For an input $x\in \mathcal{X}$ with corresponding ground truth estimation $y_{gt} \in \mathcal{M}$, the feedforward process contains two steps; firstly the deep estimation $\hat{y}^E$ is given in the embedded space $\mathbb{R}^N$, then the output manifold representation is $\hat{y}=J^{-1}(Pr(\hat{y}^E))$, where $J^{-1}$ is the inverse of extrinsic embedding $J$. 
Given that $\hat{y}^E\in\mathbb{R}^N$, where $\mathbb{R}^N$ covers the real-valued vector space of dimension $N$, $Pr$ is then dropped within the pipeline.
The training loss is computed between $\hat{y}^E$ and the extrinsic embedding $y^E_{gt}=J(y_{gt})$. Therein, the gradient used in backpropagation is computed within $\mathbb{R}^N$, leaving the original DNN architecture unchanged. 
Moreover, this implies that the transformation associated with DEMR can be directly applied to most existing DNN architectures by simply augmenting the dimensionality of the final output layer.

%
Similar to the population regression function in extrinsic regression, the estimator in embedded space $F_{NN}(x)$ as the neural network in extrinsic embedded space $\mathbb{R}^N$ aims to minimize the conditional Fr\'echet mean if it exists:
  $F_{NN}(x) = \arg\min\limits_{m\in\mathcal{M}}\int_{\mathcal{M}}L_{extr}^2(m, y)
  \Tilde{\mathcal{Q}}(\mathrm{d}y|x) 
  = \arg\min\limits_{m\in\mathcal{M}}\int_{\mathcal{M}}
   \|J(m)-J(y)\|^2 \Tilde{\mathcal{Q}}(\mathrm{d}y|x)$ 
where $L_{extr}$ is the extrinsic distance,  $\Tilde{\mathcal{Q}}(\mathrm{d}y|x)$ denotes the conditional distribution of $Y$ given $X=x$, 
and $\mathcal{Q}(\cdot|x) = \Tilde{\mathcal{Q}}(\cdot|x) \circ J^{-1}$
is the conditional probability measure on $\mathcal{M}$. Therefore, in both the training and evaluation phases of DEMR, the computation burden is partaken by the deterministic conversion $J^{-1}$, which requires no gradient computation. 
 %
 %
%
\paragraph{Reformulation of neural network for images}
The input data sample $\mathcal{I}$ is fed sequentially to the differentiable feature extractor $T(\cdot)$, where the feature extractor can be composited by various modules, such as Resnet, Pointnet, \etc,  according to the input: 
\begin{align}
  \label{eq:combination}
   \hat{y}^E=F_{NN}(\mathcal{I}) &= 
   b + P\cdot (T(\mathcal{I})) \\ \notag
   &= b + P\cdot (c_1^{\mathcal{I}},\ldots,c_n^{\mathcal{I}})^{\top}
   \\ \notag
   &\triangleq b + P\cdot C_{fm} \triangleq F_L(C_{fm})
\end{align}
where $\cdot$ indicates matrix multiplication, $b$ indicates the bias, $C_{fm}$ belongs to the vector matrix of feature maps from $T(\cdot)$, with decomposition into column vectors ${c_k^{\mathcal{I}}}$. Therefore, the DNN in DEMR serves as the composition mapping from raw input $\mathcal{I}$ to the preimage of $J^{-1}$. 
\paragraph{Projection $Pr$ onto the preimage of $J$}
For an extrinsic embedding function $J(\cdot)$, there still exists a pivotal issue that the estimation given above might not lie in the preimage $\mathcal{IM}$ of $J(\cdot)$.

Since $\hat{y}_i \in \mathbb{R}^N$, $T(\mathcal{I})$ and $\mathcal{IM}$ are all Euclidean and the projection between them can be presented as a linear transform $Pr(\cdot):$, in matrix forms. 
Other than a deterministic linear projection in extrinsic manifold regression \cite{lin2017extrinsic,lee2021robustextrinsic}, DEMR adopts a learnable projection fulfilled by linear layers within a deep framework as $F_L(\cdot)$ in Equation \ref{eq:combination}, then the output of a DNN is $Pr(F_{NN}{\mathcal{I}})\in \mathbb{R}^N$, and $J^{-1}(Pr(C_{fm})) \in \mathcal{M}$. 
Therefore, the final output of DEMR on the manifold will be
\begin{equation}
\begin{aligned}
  \hat{y} &=J^{-1}(\hat{y}^E)= J^{-1}(Pr(C_{fm}))= J^{-1}(F_L(C_{fm}))\\ \notag
  &= J^{-1}(\arg\min\limits_{q\in\Tilde{M}}\|q-\hat{F}_{NN}(\mathcal{I})\|^2)  \label{eq:Fext}     
\end{aligned}
\end{equation}

\paragraph{DIMR}
In contrast, the architecture adopted in
 \cite{lohit2017dl_mani} uses geodesic loss to train the neural network. The intrinsic geodesic distance is $d_{intr}(y_{gt}, \hat{y}) = Log_{y_{gt}}\hat{y}$ in Figure \ref{fig:DIMR}, where $Log$ is the logarithmic map of $\mathcal{M}$.
 We call it \textit{ Deep Intrinsic Manifold Representation (DIMR)} for convenience, whose model parameter set $\Theta_{NN}$ is updated with the gradients of $L_{intr}$. 


\begin{figure}
\centering
\includegraphics[width=8cm]{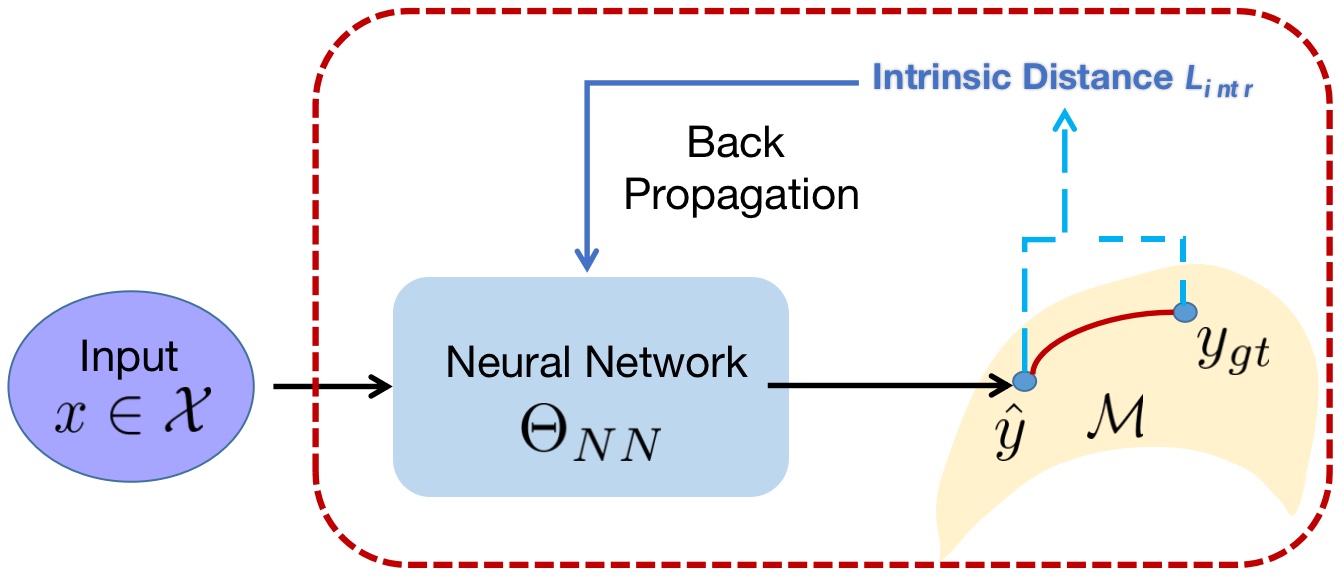}
\caption{DIMR pipeline with geodesic loss on $\mathcal{M}$, with the black arrow indicating the forward process.} 
 \label{fig:DIMR}
\end{figure}

%
\subsection{The extrinsic embedding $J$}
\label{sec:rules_embed}
The embedding $J$ is designed to preserve geometric properties to a great extent, 
which can be specified as equivariance. $J$ is considered an equivariant embedding if there is a group homomorphism $\phi: G\rightarrow GL(D, \mathbb{R})$ from $G$ to the general linear group $GL(D, \mathbb{R})$ of degree $D$ such that $\forall g\in G, q\in\mathcal{M}$. There is $J(gq)=\phi(g)J(q)$. Choosing $J$ is not unique, and the choices below are equivariant embeddings. 
Given orthogonal Group Embedding $J_O$,
for $R_1, R_2 \in O(n)$ with singular value decomposition $J_O(R_1) = U_1\Sigma_1V_1$ and $J_O(R_2) = U_2\Sigma_2V_2$ and let $J_O(R_1R_2)\triangleq K = U_1V_1^{\top}\Sigma_KU_2V_2^{\top}$, 
then there is $R_1R_2 = U_1V_1^{\top}U_2V_2^{\top}$ and 
if $\phi(R_1) = R_1\Sigma_K U_2\Sigma_2^{-1}U_2^{\top}$, there is $J_O(R_1R_2) = \phi(R_1)U_2\sigma_2V_2^{\top}=\phi(R_1)J_O(R_2)$.

The proof of Grassmannian Embedding $J_{\mathcal{G}}$ follows the same idea with proof for $O(n)$, since they are all based on matrix decomposition.
 
\subsubsection{Matrix Lie Group}
\label{sec:3.3.1}
\paragraph{9D: SVD of rank 9}
\label{sec:liegroup}
An intuitive embedding choice for the matrix Lie group is to parameterize each matrix entry. For the special orthogonal group $O(n)$ of dimension $n$, the natural embedding can be given by $J_O:O(n)\rightarrow \mathbb{R}^{n^2}$ and $J_{SO}:SO(n)\rightarrow \mathbb{R}^{n^2}$. For its inverse, $J_O^{-1}$ which aims to produce $n$ orthogonal vectors, 
where Gram Schmidt orthogonalization \cite{zhou2019continuity} and its variations are common ways to reparameterize orthogonal vectors. Singular Value Decomposition (SVD) is another convenient way to produce an orthogonal matrix \cite{levinson2020analysis_svd}:

\begin{equation}
    \mu^{ext} = \hat{F}_{NN}(x) = 
    \begin{cases}
    UV^{\top}, &det(\mu)>0\\
    UHV^{\top}, &otherwise
    \end{cases}\label{eq:svd}
\end{equation}
where $U$ and $V$ come from singular value decomposition $\mu = UDV^{\top}$ with elements arranged according to the descending order of singular value, and $H$ is a diagonal matrix $I_n$ with the last entry replaced by $-1$.
Specifically for $SO(3)$, $J^{-1}_{SO(3)}$ could also be produced by cross product.

As for the special Euclidean group, the group of isometries of the Euclidean space, \ie,  the rigid body transformations preserving Euclidean distance. 
A rigid body transformation can be given by a pair $(A, t)$ of affine transformation matrix $A$ and a translation $t$,  or it can be written in a matrix form of size $n+1$. 
For instance, a special Euclidean group can be regarded as the semidirect product of a rotation group $SO(n)$ and a translation group $T(n)$, $SE(n)= SO(n)\ltimes T(n)$. 
\paragraph{6D: cross product}
Specifically for orthogonal matrices on $SO(3)$, the invertible embedding $J$ can be more convenient via cross-product operation $\times$. 
For a 6-dimensional Euclidean vector $x = [x_a, x_b]$ as network output, where $x$ is the concatenation of $x_a$ and $x_b$, let $x_c = x_a\times x_b$, then $J^{-1}(x) = [x_a^T, x_b^T, x_c^T]$. 
\subsubsection{The Quotient Manifold of Lie Group}
\label{sec:grassmann}
The real Grassmann manifold $\mathcal{G}(m, \mathbb{R}^n)$ parameterize all $m$-dimensional linear subspaces of $\mathbb{R}^n$, which can also be defined by a quotient manifold:
$\mathcal{G}(m, \mathbb{R}^n) = O(n)/(O(m)\times O(n-m)) = SO(n)/SO(m)\times SO(n-m) = \mathcal{V}(m,n)/SO(m)$.
Since it is a quotient space, and we care more about its rank than which basis to provide, we could convert the problem of embedding $Y\in \mathcal{G}(m, \mathbb{R}^n)$ to finding mappings for $YY^{\top} \in SPD^{++}_m$, whose basis could be given by diagonal decomposition,
 which we referred to as $\mathtt{DD}$. 
 Actually, $\mathtt{DD}$ is a special case of $\mathtt{SVD}$ for a symmetric matrix.
 Thus, for a distribution $\mathcal{Q}$ defined on $\mathcal{G}(m, \mathbb{R}^n)$, given $\mu = \hat{F}(x)\in \mathbb{R}^{n\times m}$ and its diagonal decomposition (denoted by $\mathtt{DD}$) $\mu = U\Sigma U^{\top}$, the inverse embedding $J^{-1}_{\mathcal{G}}$ comes from $\mu^{ext} = J^{-1}_{\mathcal{G}}(\mu) = USU^{\top}$, and the first $m$ column vectors of $U$ constitute the subspace. 

\subsection{DEMR as a generalization of previous research}
\label{sec:dnn_struct}
%
%
\paragraph{DNN with Euclidean output}
When the output space is a vector space,  $J(\cdot)$ becomes an identity mapping. Thus, DNN with Euclidean output can be treated as a degenerate form of DEMR.
The output from $F_NN$ is linearly transformed from the subspace spanned by the base $C_{fm}$, causing its failure for new extracted features $c'\notin C_{fm}$.
It is the corresponding failure case in \cite{zhou2019continuity}, when the dimension of the last layer equals the output dimension. 
%
\paragraph{Absolute/relative pose regression}
%
From Equations \ref{eq:combination}, the estimation $F_{NN}(\mathcal{I})$ is linearly-transformed from the subspace spanned by ${c_k^{\mathcal{I}}}$.
Hence, in typical DNN-based pose regression tasks, the predicted pose $\hat{y}$ can be seen as a linear combination of feature maps extracted from input poses, resulting in the failure of extrapolation for unseen poses in the test set. 
It is a common problem because training samples are often given from limited poses. 
\cite{zhou2019continuity} suggests that it is 
the discontinuity in the output space incurs poor generalization. 
Indeed, a better assumption for APR is that the pose estimation lies on $SE(3)$, which is composed of a rotation and a translation. The manifold assumption renders the estimator more powerful in continuous interpolation and extrapolation from input poses. This is because the continuity and symmetry to $SE(3)$ help a lot in the deep learning task. The detailed analysis is in Section \ref{sec:3.3.1} and validation in Section \ref{sec:exp1}
Thereon APR on $SE(3)$ can be regarded as a particular case solved by DEMR; and \cite{sattler2019understanding,zhou2019continuity} revealed parts of the idea from DEMR. 

\section{Analysis}
\label{sec_analysis}
In line with the experimental setup, the analysis is performed on the special orthogonal group, its quotient space, and the special Euclidean group \footnote{All the proofs are included in the Appendix}.
\subsection{Feasibility of DEMR}
Before conducting optimization in extrinsic embedding space $\mathbb{R}^N$, the primary misgiving consists of whether the geometry in extrinsic embedded space properly reflects the intrinsic geometry of $\mathcal{M}$.  

Apparently, extrinsic embedding $J(\cdot)$ is a diffeomorphism preserving geometrical continuity, which is advantageous for extrinsic embeddings. Since we want to observe conformance between $L_{extr}$ and $L_{intr}$, it's natural to bridge distances with smoothness.  Firstly, we need the diffeomorphism between manifolds to be bilipshitz.
\begin{lemma}
\label{lem:lipshitz}
Suppose that $\mathcal{M}_1, \mathcal{M}_2$ are smooth and compact Riemannian manifolds,$f:\mathcal{M}_1 \rightarrow \mathcal{M}_2$ is a diffeomorphism. Then $f$ is bilipschitz \wrt the Riemannian distance.
\end{lemma}
Then we show in \ref{prop:conformance} the conformity of extrinsic distance and intrinsic distance, which enables indirectly representing an intrinsic loss in Euclidean spaces.
\begin{proposition}
\label{prop:conformance}
For a smooth embedding $J:\mathcal{M}\rightarrow\mathbb{R}^m$, where the n-manifold $\mathcal{M}$ is compact, and its metric is denoted by $\rho$, the metric of $\mathbb{R}^m$ is denoted by $d$. For any two sequences $\{x_k\}, \{y_k\}$ of points in $\mathcal{M}$ and their images $\{J(x_k)\}, \{J(y_k)\}$, if $\lim\limits_{n\rightarrow\infty}d(J(x_n), J(y_n))=0$ then $\lim\limits_{n\rightarrow\infty}\rho(x_n, y_n)=0$.
\end{proposition}

\subsection{Asymptotic MLE}
In this part, we demonstrate that DEMR for $SO(3)$ is the approximate maximum likelihood estimation (MLE) of the $SO(3)$ response, and for Grassmann manifold $\mathcal{G}(m, \mathbb{R}^n)$ DEMR is the MLE of the Grassmann response.

To be noticed, we adopt a new error model in conformity with DEMR.
In previous work such as \cite{levinson2020analysis_svd}, 
the error noise matrix $N$ is assumed to be filled with random entries $n_{ij}\sim N(0,\sigma)$, which is not rational,  because $N$ also lies on $\mathcal{M}$ and there are innate structures between the entries of $N$.
Here we assume $N\in\mathcal{M}$, so the probability $Q$ on $\mathcal{M}$ should be established first.
\subsubsection{Lie group for transformations}As \cite{bourmaud2015EKF} suggested,
we consider the connected, unimodular matrix Lie group, including the most frequently used categories in computer vision: $SE(3)$, $SO(3)$, $SL(3)$, \etc.  Since $SO(3)$ is a degenerate case of $SE(3)$ without translation, here we consider the concentrated Gaussian distribution on $SE(3)$. The probability density function (pdf) takes the form $
    P(x;\Sigma) = \alpha \exp{
    -\frac{1}{2}(
    [\log_{\mathcal{M}}(x)]_{\mathcal{M}}^{\vee\top}\Sigma^{-1}[\log_{\mathcal{M}}(x)]_{\mathcal{M}}^{\vee}
    )
    }
$, 
 where $\alpha$ is the normalizing factor, $x\in SO(3)$ and the covariance matrix $\Sigma$ is positive definite. 
 
 Maps $[\cdot]^{\wedge}$and $[\cdot]^{\vee}$ are linear isomorphism, re-arranging the Euclidean representations into anti-symmetric matrix and back \footnote{
    $hat [\cdot]^{\wedge}: \mathbb{R}^3 \rightarrow so(3); \mathbf{\theta}\rightarrow \mathbf{\theta} = [\mathbf{\theta}]_{\times}, 
    vee[\cdot]^{\vee}:so(3)\rightarrow \mathbb{R}^3; [\mathbf{\theta}]_{\times}^{\wedge}=\mathbf{\theta}
$ where $[\cdot]_{times}$ indicates the antisymmetric matrix form of the vector.
}. 
For a vector $\theta = [\theta_1, \theta_2, \theta_3]^{\top}$,  there is $\mathbf{\theta}^{\wedge} = [\mathbf{\theta}]_{\times}$ and $[\mathbf{\theta}]_{\times}^{\vee}=\mathbf{\theta}$.

\begin{proposition}
\label{prop:so3_MLE}
$J^{-1}_{SO}:\mathbb{R}^9\rightarrow SO(3)$  gives an approximation of MLE of rotations on $SO(3)$, if assuming $\Sigma = \sigma I$,  $I$ is the identity matrix and $\sigma$ an arbitrary real value.
\end{proposition}

DEMR on special Euclidean Group $SE(3)$ also approximately provides maximum likelihood estimation, sharing similar ideas of proof with $SO(3)$.   
\begin{proposition}
\label{prop:se3_MLE}
$J^{-1}_{SE}:\mathbb{R}^9\rightarrow SE(3)$ , where the rotation part comes from $\mathtt{SVD}$ ((see Supplementary Material) gives an approximation of MLE of transformations on $SE(3)$, if assuming $\Sigma = \sigma I$,  $I$ is the identity matrix and $\sigma$ an arbitrary real value.
\end{proposition}
\begin{proof}
Thus, for $x \sim \mathcal{N}_{\mathcal{G}}(\mu, \Sigma)$, there is $x=\mu\exp_{\mathcal{M}}([N]^{\vee}_{\mathcal{M}}) $, and the simplified log-likelihood function is
\begin{equation}
\begin{aligned}
 & L(Y;F_{\Theta_{NN}},\Sigma) = \\ \notag
 & [\log_{\mathcal{M}}(F_{\Theta_{NN}}-Y)]_{\mathcal{M}}^{\vee\top}\Sigma^{-1}[\log_{\mathcal{M}}(F_{\Theta_{NN}}-Y)]_{\mathcal{M}}^{\vee}  
 \end{aligned}
\end{equation}
with $\log_{\mathcal{M}}(F_{\Theta_{NN}}-Y)]_{\mathcal{M}}^{\vee}$ set to be $\epsilon$, if the pdf focused around the group identity, i.e. the fluctuation of $\epsilon$ is small, the noise $\epsilon$ 's distribution could 
be approximated by $\mathcal{N}_{\mathbb{R}^3}(\mathbf{0}_{3\times1},\Sigma)$ on $\mathbb{R}^3$. Then there is 
$\arg\max\limits_{Y\in SO(3)}L(Y;F_{\Theta_{NN}},\Sigma) \approx \arg\min\limits_{Y\in SO(3)}(F_{\Theta_{NN}}-Y))^{\top}(F_{\Theta_{NN}}-Y))=\arg\min\limits_{Y\in SO(3)}\|F_{\Theta_{NN}}-Y\|_F^2$. 
\end{proof}
\subsubsection{Grassmann for subspaces}
\label{sec:grass_MLE}
One of the manifold versions of Gaussian distribution on Grassmann and Stiefel manifolds is Matrix Angular Central Gaussian (MACG). However, the matrix representation of linear subspaces shall preserve the consistency of eigenvalues across permutations and sign flips, thus we resort to the symmetric $UU^T$ on Symmetric Positive Definite (SPD) manifold for $U\in\mathcal{G}(m,\mathbb{R}^N)$.  
%

Similarly,  the manifold output on 
the error model is modified to be $F_{\Theta_{NN}} = YY^{\top} + N$, where both $YY^{\top}$ and $N$ are semi-positive definite matrix lying on SPD manifold. The Gaussian distribution extended on SPD manifold, for a random $2th-$order tensor $X\in \mathbb{R}^{n\times n}$ is
\begin{equation}\label{eq:grasspdf}
    P(X;M,S) = \frac{1}{\sqrt{(2\pi)^n}|\mathcal{S}|}\exp{\frac{1}{2}(X-M)^T\mathcal{S}^{-1}(X-M)}
\end{equation}
with mean $M\in \mathbb{R}^{m\times m}$ and $4th-$order
covariance tensor $\mathcal{S}\in \mathbb{R}^{m\times m\times m\times m}$ inheriting symmetries from three dimensions.
For the inverse of Grassmann manifold extrinsic embedding $J_{\mathcal{G}}^{-1}$, DEMR provides MLE, and the proofs are given in the Supplementary Material.
\begin{proposition}
\label{prop:grass_MLE}
DEMR with $J^{-1}_{\mathcal{G}}$  gives MLE of element on $\mathcal{G}(m.\mathbb{R}^n)$, if assuming $\Sigma$ is an identity matrix.
\end{proposition}

\subsection{Generalization Ability}
\subsubsection{Failure of DNN with Euclidean output space}
In light of the analysis in \ref{sec:dnn_struct}, the output space is produced by a linear transformation on the convolutional feature space spanned by the extracted features.
The feature map extracted by a neural network organized as matrices in Equation (\ref{eq:combination}) the source to form the basis.
Denoting the linear subspace spanned by  the feature map basis ${c_1^{\mathcal{I}},\ldots,c_n^{\mathcal{I}}}$ to be $\mathbb{R}_{feature}$ and let $\mathbb{R}_{output}$ be the low-dimensional output space spanned by $\{o_1,\ldots,o_{n_o}\}$, where $n_o$ denotes the output dimension, then $F_{NN}(\mathcal{I})\in \mathbb{R}_{output}$. 
For a new test input $\mathcal{I}'$,  its extracted feature map belongs to the complementary space of $ \mathbb{R}_{feature}$, there is $F_{NN}(\mathcal{I}')\notin \mathbb{R}_{output}$. This accounts for the failure of some DNN models with Euclidean output.  

\subsubsection{Representation power of structured output space}
This section studies the enhancement of the representational power of DEMR, when the output space is endowed with geometrical structure.
As a linear action,
one representation of a Lie group is a smooth group homomorphism $\Pi:G\rightarrow GL(V)$ on the $n$-dimensional vector space $V$, where $GL(V)$ is a general linear group of all invertible linear transformations. 
\begin{proposition}
\label{prop:so_generalization}
Any element of dimension $n$ on $SO(n)$ belongs to the image of $J^{-1}_{SO(n)}$ from known rotations within a certain range, 
if the Euclidean input of $J^{-1}_{SO(n)}$ is of more than $n$ dimensions.
\end{proposition}
\begin{corollary}
\label{col:SEn}
Any element of dimension $n$ on $SE(n)$ belongs to the image of $J^{-1}_{SE(n)}$ from known rotations within a certain range, 
if the Euclidean input of $J^{-1}_{SE(n)}$ is of more than $n$ dimensions.
\end{corollary}
Then for a linear representation of Lie group in matrix form, given a set of basis on $V$, we show that the output of DEMR better extrapolates the input samples than common deep learning settings with unstructured output, which resolves the problem raised in \cite{sattler2019understanding}. 

\section{Experiments}
\label{sec:exp}

In this section, we demonstrate the effectiveness of applying extrinsic embedding to deep learning settings on two representative manifolds in computer vision.
The experiments are conducted from several aspects below:
\begin{itemize}
    \item Whether DEMR takes effects in improving the performance of certain tasks?
    \item Whether the geometrical structure boosts model performance facing unseen cases, \eg, the ability to extrapolate training set.
\item Whether extrinsic embedding yields valid geometrical restrictions.
\end{itemize}
 The validations are conducted on two canonical manifold applications in computer vision
\subsection{Task I: affine motions on $SE(3)$}
\label{sec:exp1}
Estimating the relative position and rotation between two point clouds has a wide range of downstream applications. 
The reference and target point clouds 
 $P_R, P_t \in \mathbb{R}^{N\times3}$ are in the same size and shape, with no scale transformations.
The relative translation and rotation can be arranged separately in vectors 
or together in a matrix lying on $SE(3)$.
\subsubsection{Experimental Setup}
\paragraph{Training Detail}
During training, at each iteration stage, a randomly chosen point cloud from $2,290$ airplanes is transformed by randomly sampling rotations and translations in batches. The rotations are sampled according to the models, \ie, the models producing axis angles are fed with rotations sampled from axis angles, and so on. 
%

\paragraph{Comparison metrics}
To validate the ability of the model to preserve geometrical structures, geodesic distance is the opted metric at the testing stage. 
\cite{zhou2019continuity,levinson2020analysis_svd} uses minimal angular difference to evaluate the differences between rotations, 
which is not entirely compatible with Euclidean groups, since it calculates the translation part and rotation part separately. 
For two rotations $R_1, R_2$ and $R' = R_1R_2^{-1}$ with its trace to be $tr(R')$, there is $L_{angle}=\cos^{-1}((tr(R')-1)/2)$.
For intrinsic metric, the geodesic distance between 
two group elements are defined with Frobenius norm for matrices
 $d_{int}(M_1, M_2)= \|\log(M_1^{-1}M_2)\|$, 
 where $\log(M_1^{-1}M_2)$ indicates the logarithm map on the Lie group.
 For extrinsic metric, we take Mean Squared Loss (MSE) between $J_{SE}(M_1)$ and $J_{SE}(M_2)$.
\begin{table}[H]
\centering
\caption{On the task of estimating relative affine transformations on $SE(3)$ for point clouds, the results are reported by the geodesic distance between the estimation and the ground truth. The best results across models are emphasized, where the smallest errors come from extrinsic embedding.}
\begin{tabular}{c|ccc}
    \toprule
    Mode&avg&median&std\\
    \midrule
         euler&26.83 & 15.54&32.23\\
         axis&31.96 & 10.66&27.75\\
         6D&10.28 &\textbf{5.36}&\textbf{21.54} \\
         9D&\textbf{10.23} &6.09&25.30\\
     \bottomrule  
    \end{tabular}
\label{tab:full_SE3}
\end{table}
 For models with Euler angle and axis-angle output,  the predicted rotation and translation are organized in Euclidean form, and they are reorganized into Lie algebra $se(3)$. Finally, the distances are calculated in $SO(3)$ by exerting an exponential map on their $se(3)$ forms.

\paragraph{Architecture detail}
Similar as \cite{zhou2019continuity,levinson2020analysis_svd}, 
the backbone for point cloud feature extraction is composed of 
a weight-sharing Siamese architecture containing two simplified PointNet Structures \cite{qi2017pointnet}
$\Phi_i:\mathbb{R}^{N\times3}\rightarrow \mathbb{R}^{1024}$
where $i=1,2$. 
After extracting the feature of each point with an MLP, 
both $\Phi_1$ and $\Phi_2$ utilize max pooling to produce a single vector $z_1, z_2$ respectively, 
as representations of features across all points in $P_r, P_t$. Concatenating $z_1, z_2$ to be $z$, 
another MLP is used for mapping the concatenated feature to the final higher-order Euclidean representation 
$Y_{\mathbb{E}}$. 
Finally, Euler angle quaternion coefficients and axis-angle representations will be directly obtained from the last linear layer of the backbone network of dimension, $3, 3, 4$ respectively.
For manifold output, the $SE(3)$ representation will be given by both cross-product and SVD, with details in Supplementary materials.
The translation part is produced by another line of $\mathtt{fc}$ layer of 3 dimensions.
 Because MSE computation between 
two isometric matrices are based on each entry,  total MSE loss is the sum of rotation loss
 and translation loss.
%
\paragraph{Estimation accuracy and conformance}
To validate the necessity of structured output space, 
firstly we compare the affine transformation estimating network 
with different output formations. 
The translation part naturally takes the form of a vector 
and the rotation part can be in the form of Euclidean representations 
and $SO(3)$ in matrix forms.
%
The dataset is composed of generated point cloud pairs with random transformations, where the rotations are uniformly sampled from various representations, namely, Euler angle, axis-angle in $3$-dimensional vector space, and $SO(3)$. The translation part is randomly sampled from a standard normal distribution. 

The advantage of manifold output produced by baseline DEMR is revealed in Table \ref{tab:full_SE3} and Figure \ref{fig:percentile}, output via extrinsic embedding takes the lead across three statistics.
\begin{figure}[H]
    \centering
\includegraphics[width=5cm]{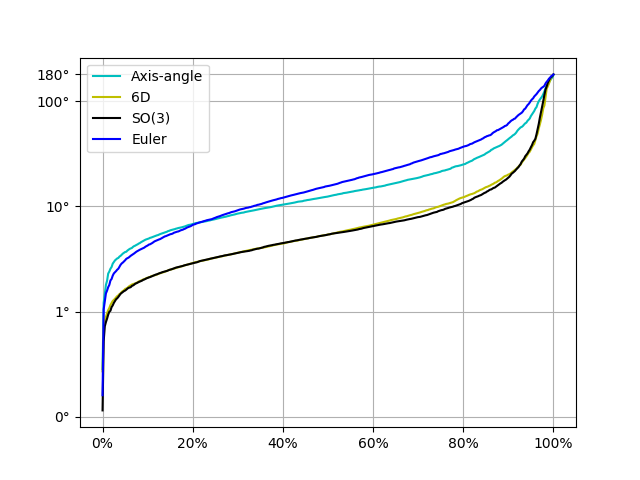} 	
    \caption{The cumulative distributions comparison of position errors for the pose regression task on $SE(3)$.}
	\label{fig:percentile}
\end{figure}
\begin{table*}
    \centering
    \small
    \caption{Comparison of different embeddings. DEMR takes the lead in predicting accuracy, with two types of embedding functions, namely, 6D and 9D.  Even when the ratio of unseen cases increases, DEMR has better performance all along.}
    \scalebox{0.85}{
    \begin{tabular}{c|ccccccccccccccc}
    \toprule
    &  
    \multicolumn{3}{c}{10\%}&\multicolumn{3}{c}{20\%}&\multicolumn{3}{c}{40\%}&\multicolumn{3}{c}{60\%}&\multicolumn{3}{c}{80\%}\cr
    \cmidrule(lr){2-4}\cmidrule(lr){5-7}\cmidrule(lr){8-10}\cmidrule(lr){11-13}\cmidrule(lr){14-16}
    &avg&med&std &avg&med&std &avg&med&std &avg&med&std &avg&med&std\cr
        euler & 77.81&54.42&77.52 &35.58&13.07&48.94 &28.82&17.02&30.89&27.22&16.27&35.52 &27.22&16.27&31.75 \\
        axis &75.57 &50.71 &78.60 &36.95 &12.38 &54.40&20.73&12.51&27.63&21.95&12.75&27.15&18.96&10.66&24.75\\
        6D&71.09&\textbf{29.39}&80.09&\textbf{17.12}&\textbf{5.25} &40.54&\textbf{10.85}&\textbf{5.07}&\textbf{19.60}&\textbf{9.76}&\textbf{4.93}&\textbf{19.60}&10.29&\textbf{5.01}&\textbf{21.08} \\
        9D&\textbf{70.72}&30.99 &\textbf{63.18}&19.83&5.62&\textbf{39.22}&12.91&6.28&22.66&12.24&6.40&24.18&\textbf{10.23}&5.31&22.60 \\
    \toprule
    \end{tabular}
    }    
    \label{tab:generalization}
\end{table*}
\begin{table*}
    \centering
        \small
        \caption{Test result reported in average $D_{\mathcal{G}}$ with an example result of a test sample, and the epoch at which the model converged.}
    \begin{tabular}{c|ccc}
    \toprule
        Input & GrassmannNet \cite{lohit2017dl_mani} & DIMR & DEMR\\
        \midrule

         \begin{minipage}
         [b]{0.3\columnwidth}
		\centering
		\raisebox{.2\height}{\includegraphics[width=\linewidth]{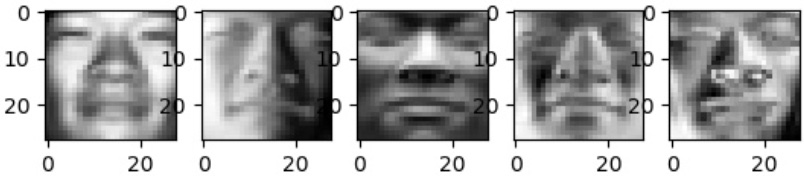}}
        \end{minipage}  
        &
         \begin{minipage}
         [b]{0.3\columnwidth}
		\centering
		\raisebox{.2\height}{\includegraphics[width=\linewidth]{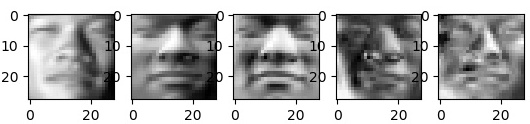}}
        \end{minipage}
        &
        \begin{minipage}
         [b]{0.3\columnwidth}
		\centering
		\raisebox{.2\height}{\includegraphics[width=\linewidth]{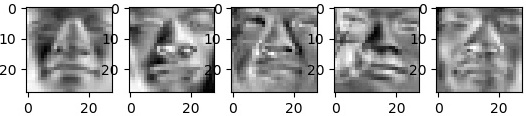}}
        \end{minipage}  
        &
        \begin{minipage}
         [b]{0.3\columnwidth}
		\centering
		\raisebox{.2\height}{\includegraphics[width=\linewidth]{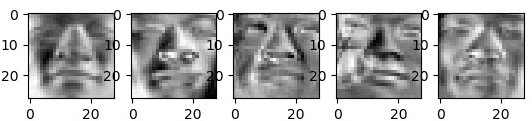}}
        \end{minipage}  
\\
        avg $D_{\mathcal{G}}(U_{gt}, U_{out})$& 9.6826& 3.4551&3.4721 \\
     epoch & 320 & 680&140\\
     
    \bottomrule   
    \end{tabular}
    \label{tab:grassresult}
\end{table*}

\paragraph{Generalization ability}
To evaluate DEMR over its improvement in generalization ability in comparison with unstructured output space, the training set 
only contains a small portion of the whole affine transformation space while the test set encompasses affine transformations 
sampled from the whole rotation space.
To be fair in the sampling step across the compared representations, the sampling process is conducted uniformly in 
the representations spaces respectively.  
 For models with axis-angle and Euler angle output, the input point cloud pair is constructed with rotation representation whose entries
  sampled $\mathrm{iid}$ from Euclidean ranges.
 For Euler representations, the lower bounds of the ranges are  $\pi$, 
 and the upper bounds are $\pi$ multiplied by $90\%, 80\%, 60\%, 40\%,20\%$, 
 while the test set is sampled from $[-\pi, \pi]$. 
 The training set with axis-angle representations is also obtained by sampling Euclidean ranges. 
 When constructing a training set on $SO(3)$, the uniform sampling process is conducted in the same way with axis-angle sampling, the segments are taken in ratios in the same way with the two aforementioned settings,  and then exponential mapping is used to yield random samples on $SO(3)$.
Facing unseen cases, DEMR has better extrapolation capability compared to plain deep learning settings with unstructured output space.  
As shown in Table \ref{tab: generalization}, results from deep extrinsic learning show great competence despite the portion size of the training set. 

\subsection{Task II: illumination subspace, the Grassmann manifold}
Changes in poses, expressions, and illumination conditions are inevitable in real-world applications,
but seriously impair the performance of face recognition models.
It is universally acknowledged that an image set of a human face under different conditions
 lies close to a low-dimensional Euclidean subspace. 
After vectorizing images of each person into one matrix in dataset Yale-B \cite{YaleB}, 
experiments reveal that the top $d=5$ principal components (PC) generally capture more than $90\%$ of the singular values.
As an essential application of DEMR, we take the feature extraction CNN as $T(\cdot)$ in Equation \ref{eq: combination}.
in high-dimensional Euclidean space, and obtain the final result by
inversely mapping the estimation to the Grassmann manifold.

\subsubsection{Experimental Setup}
\paragraph{Architectural detail}
A similar work sharing the same background comes from \cite{lohit2017dl_mani}, which assumes the output of the last layer in the neural network to lie on the matrix manifolds or its tangent space. We take GrassmannNet in the former assumption in \cite{lohit2017dl_mani} as the baseline, where the training loss function is set to be Mean Squared Loss. For all the models to be compared, the ratio of the training set is $0.8$, and all of the illumination angles are adopted.
In addition, the data preparation step also follows \cite{lohit2017dl_mani},  and other training details are recorded in supplementary materials.  
On extended Yale Face Database B, each grayscale image of $168\times192$ is firstly resized to $256\times256$ as the input of CNN, while for subspace ground truth, they are resized to $28\times28=784$ for convenience of computation.  
 \paragraph{Dataset processing}
The input face image $I_i^j$ indicates the $i_{th}$ face under the $j_{th}$ illumination condition, where $j=1,\ldots, 64$ for Yale B dataset.  If we adopt the top $d$ PCs, the output is assumed to be in the form of $\{E_i^1, E_i^2,\ldots, E_I^d\}$, $k=1\ldots, d$ and $<E_i^k, E_i^l>=\delta_{kl}$ where $\delta_{kl}$ is a Kronecker delta function. Each PC is rearranged as a vector, thus $vec(E_i^k)$ is of size $784\times1$ and we define $U_i = [vec(E_i^1), vec(E_i^2) \ldots vec(E_i^d)]$, and the output of the network is $U_iU_i^T$.
Most of the intrinsic distances defined on the Grassmannian manifold are based on principal angles (PA), for $P_1, P_2\in \mathcal{G}(m, \mathbb{R}^n)$,
with SVD $P_1^{T}P_2=USV^{T}$ where $S=diag\{cos(\theta_1),cos(\theta_2),\ldots,cos(\theta_m)\}$, PAs of Binet-Cauchy (BC) distance is $1-\Pi_{k=1}^K\cos^2\theta_k$ 
and PAs of Martin (MA) distance is $\log\Pi_{k=1}^K(\cos^2\theta_k)^{-1}$.
For comparison of effectiveness and training advantage,  the deep intrinsic learning takes $D_{\mathcal{G}}(U_{gt}, U_{out}) = $ as both the training loss and testing metric, where $U_{gt}$ and $U_{out}$ indicate the ground truth and the network output respectively.

\paragraph{Estimation accuracy and conformance}
The results of illumination subspace estimation are given in Table \ref{tab:grassresult}. Measured by the average of geodesic losses, DEMR architecture achieves nearly the same performance as DIMR while drastically reducing training time till convergence. 
With batch size set to be 10, we record training loss after every 10 epochs and report approximately the epoch at which the training loss converged in the third row in Table \ref{tab:grassresult}.
Besides, both DIMR and DEMR surpass GrassmannNet in the task of illumination subspace estimation, and DEMR still takes fewer epochs for training than GrassmannNet.
In the first row of Table \ref{tab:grassresult}, on one of the test samples, the results from GrassmannNet, DIMR, and DEMR are displayed. In addition, the conformance of extrinsic and intrinsic metrics can be referred to as supplementary materials.

\section{Conclusion}
This paper presents Deep Extrinsic Manifold Representation (DEMR), which incorporates extrinsic embedding into DNNs to circumvent the direct computation of intrinsic information. Experimental results demonstrate that retaining the geometric structure of the manifold enhances overall performance and generalization ability in the original tasks. Furthermore, extrinsic embedding exhibits superior computational advantages over intrinsic methods. Looking ahead, the prospect of a more unified formulation for extrinsic embedding techniques within deep learning settings, characterized by parameterized approaches, is envisioned.

\bibliography{ref}
\bibliographystyle{icml2024}

\section*{Appendix}
\subsubsection{Proofs in Section \ref{sec_analysis}}
\paragraph{Consistency of extrinsic loss and intrinsic loss}
\paragraph{Proof of lemma \ref{lem:lipshitz}}
\begin{lemma}
Suppose that $\mathcal{M}_1, \mathcal{M}_2$ are smooth and compact Riemannian manifolds,$f:\mathcal{M}_1 \rightarrow \mathcal{M}_2$ is a diffeomorphism. Then $f$ is bilipschitz \wrt the Riemannian distance.
\end{lemma}
\begin{proof}
For $\mathcal{M}_1$ and $\mathcal{M}_2$ with smooth metrics $\rho_1$ and $\rho_2$ respectively, 
considering the smooth and continuous map  $\mathrm{d}f:U_1\rightarrow T\mathcal{M}_2$, where $U_1$ indicates the unit tangent bundle of $\mathcal{M}_1$  and $T\mathcal{M}_2$ represents the tangent space, then the function $\phi(u)=|df(u)|, u\in U_1$ is continuous too. Since $U_1$ is compact, there exists the maximum $C$. Denoting the length of path $r:[a,b]\rightarrow \mathcal{M}_1$ to be $L(\cdot)$, then there is $
L(f\circ r) = \int_a^b|(f\circ r)'|\mathrm{d}t\leq C\int_a^b|c'(t)|\mathrm{d}t=CL(r)$
thus $\rho_2(f(t_1), f(t_2))\leq C\rho_1(t_1, t_2), t_1, t_2\in \mathcal{M}_1$, then $f$ is $C-$ Lipshitz and $f^{-1}$ is $\frac{1}{C}-$ Lipshitz.
\end{proof}

\paragraph{Proof of proposition
\ref{prop:conformance}}
%
The extrinsic metric could reflect the tendency of the intrinsic metric. 
Here we pay more attention to consistency when the losses are small because consistency is crucial in extrinsic loss convergence.
To be more specific, for a smooth embedding $J:\mathcal{M}\rightarrow\mathbb{R}^m$, where the n-manifold $\mathcal{M}$ is compact, and its metric is denoted by $\rho$, 
the metric of $\mathbb{R}^m$ is $d$. The minimal change of $d$ suggests the minimal change of $\rho$, which will be a direct result if $J$ is bilipshitz. Here the condition of compactness of the embedded space will be relaxed, where $J$ is only desired to be locally bilipshitz, while the proof follows the same idea of the lemma above.
\begin{proposition}
For a smooth embedding $J:\mathcal{M}\rightarrow\mathbb{R}^m$, where the n-manifold $\mathcal{M}$ is compact, and its metric is denoted by $\rho$, the metric of $\mathbb{R}^m$ is denoted by $d$. For any two sequences $\{x_k\}, \{y_k\}$ of points in $\mathcal{M}$ and their images $\{J(x_k)\}, \{J(y_k)\}$, if $\lim\limits_{n\rightarrow\infty}d(J(x_n), J(y_n))=0$ then $\lim\limits_{n\rightarrow\infty}\rho(x_n, y_n)=0$.
\end{proposition}
\begin{proof}
Given that every compact submanifold $Y_s = J(\mathcal{M}) \subset Y$ has positive normal injectivity radius, where $Y\subset \mathbb{R}^N$,  there exists a positive constant $r$,  such that 
$B_r v_{Y_s}=\{v\in v_{Y_s}:|v|\leq r\}$ with 
$v_{Y_s}$ to be the normal bundle of $Y_s$ in $Y$, 
and the normal exponential map $\exp_{Y_s}:v_{Y_s}\rightarrow Y$ is a diffeomorphism onto its image $S = \exp(B_r(v_{Y_s}))$. 
Because the exponential map $\exp_{Y_s}$  of $\mathcal{M}$ is a diffeomorphism,  $S$ is an open neighborhood of $Y_s$, and the inverse of $\exp_{Y_s}$,  $\log_{Y_s}$ is smooth.  
Let the retraction $g=p \circ \log_{Y_s} :S\rightarrow Y_s$ to be  $\log_{Y_s}\circ \exp_{Y_s}$ as the composition of the projection $p: Y_s\rightarrow Y$ and $\log_{Y_s}$,  the closure of $B_{r'}v_{Y_s}$ to be  $\Bar{B}_{r'}v_{Y_s}$ for $0<r'<r$, then its image under $\exp_{Y_s}$ would be a compact submanifold with boundary $Z\subset Y_s$. Hence there exists a real constant $C_1 \geq 0$ such that $g$ is $C_1$-Lipshitz on $Z$. 
 From lemma 1 there exists a real constant $C_2\geq 0$ such that $J$ is $C_2$-lipshitz.
Moreover, since $g$ is the right inverse of the inclusion map $i: Y_s\rightarrow Y$, it follows that $i$ is $C_1$-bilipshitz, when restricted to the sufficiently small ball $B_{\epsilon}$, where $\epsilon < r'$.

Next, define the map $\phi(y_1, y_2) = \frac{d_{Y_s}(y_1, y_2)}{d_{Y}(y_1, y_2)}$, where $(y_1, y_2)\in Y_s\times Y_s$ and $y_1\neq y_2$, and $\phi$ is continuous since the Euclidean distance function $d_{Y_s}$ and $d_Y$ is continuous. For a compact subset $S_{comp}\subset Y_s\times Y_s$, there is a real value $C_3>0$ such that $\phi$ on $S_{comp}$ is bounded by $C_3$. Finally, J is locally $C_{max}$-bilipshitz, where $C_{max}=\max(C_1, C_2, C_3)$. Further, for any two sequences $\{x_k\}, \{y_k\}$ of points in $\mathcal{M}$ and their images $\{J(x_k)\}, \{J(y_k)\}$, if $\lim\limits_{n\rightarrow\infty}d(J(x_n), J(y_n))=0$, there is $\rho(x_n, y_n) \leq C_{max}d(J(x_n), J(y_n))$, $\lim\limits_{n\rightarrow\infty}\rho(x_n, y_n)=0$.
\end{proof}

\subsubsection{Asymptotic maximum likelihood estimation (MLE)}
\paragraph{Proof of proposition \ref{prop:so3_MLE}}
\begin{proposition}
$J^{-1}_{SO}:\mathbb{R}^9\rightarrow SO(3)$  gives an approximation of MLE of rotations on $SO(3)$, if assuming $\Sigma = \sigma I$,  $I$ is the identity matrix and $\sigma$ an arbitrary real value.
\end{proposition}
\begin{proof}
For $x \sim \mathcal{N}_{\mathcal{G}}(\mu, \Sigma)$, there is $x=\mu\exp_{\mathcal{M}}([N]^{\vee}_{\mathcal{M}}) $, and the simplified log-likelihood function is
\begin{equation}
 L(Y;F_{\Theta_{NN}},\Sigma) = [\log_{\mathcal{M}}(F_{\Theta_{NN}}-Y)]_{\mathcal{M}}^{\vee\top}\Sigma^{-1}[\log_{\mathcal{M}}(F_{\Theta_{NN}}-Y)]_{\mathcal{M}}^{\vee}   
\end{equation}
with $\log_{\mathcal{M}}(F_{\Theta_{NN}}-Y)]_{\mathcal{M}}^{\vee}$ set to be $\epsilon$, if the pdf focused around the group identity, i.e. the fluctuation of $\epsilon$ is small, the noise $\epsilon$ 's distribution could 
be approximated by $\mathcal{N}_{\mathbb{R}^3}(\mathbf{0}_{3\times1},\Sigma)$ on $\mathbb{R}^3$. Then there is 
$\arg\max\limits_{Y\in SO(3)}L(Y;F_{\Theta_{NN}},\Sigma) \approx \arg\min\limits_{Y\in SO(3)}(F_{\Theta_{NN}}-Y))^{\top}(F_{\Theta_{NN}}-Y))=\arg\min\limits_{Y\in SO(3)}\|F_{\Theta_{NN}}-Y\|_F^2$. 
\end{proof}

\paragraph{Proof of proposition \ref{prop:se3_MLE}}
\begin{proposition}
$J^{-1}_{SE}:\mathbb{R}^9\rightarrow SE(3)$ where the rotation part comes from $\mathtt{SVD}$,  gives an approximation of MLE of transformations on $SE(3)$, if assuming $\Sigma = \sigma I$,  $I$ is the identity matrix and $\sigma$ an arbitrary real value.
\end{proposition}
\begin{proof}
Thus, for $x \sim \mathcal{N}_{\mathcal{G}}(\mu, \Sigma)$, there is $x=\mu\exp_{\mathcal{M}}([N]^{\vee}_{\mathcal{M}}) $, and the simplified log-likelihood function is
\begin{equation}
 L(Y;F_{\Theta_{NN}},\Sigma) = [\log_{\mathcal{M}}(F_{\Theta_{NN}}-Y)]_{\mathcal{M}}^{\vee\top}\Sigma^{-1}[\log_{\mathcal{M}}(F_{\Theta_{NN}}-Y)]_{\mathcal{M}}^{\vee}   
\end{equation}
with $\log_{\mathcal{M}}(F_{\Theta_{NN}}-Y)]_{\mathcal{M}}^{\vee}$ set to be $\epsilon$, if the PDF focused around the group identity, i.e. the fluctuation of $\epsilon$ is small, the noise $\epsilon$ 's distribution could 
be approximated by $\mathcal{N}_{\mathbb{R}^3}(\mathbf{0}_{3\times1},\Sigma)$ on $\mathbb{R}^3$. Then there is 
$\arg\max\limits_{Y\in SO(3)}L(Y;F_{\Theta_{NN}},\Sigma) \approx \arg\min\limits_{Y\in SO(3)}(F_{\Theta_{NN}}-Y))^{\top}(F_{\Theta_{NN}}-Y))=\arg\min\limits_{Y\in SO(3)}\|F_{\Theta_{NN}}-Y\|_F^2$. 
\end{proof}

\paragraph{Proof of proposition \ref{prop:grass_MLE}}
For $\mathcal{G}(m, \mathbb{R}^n)$, the probability density function of MACG is
\begin{align}
    P(x;R_{\mu}, \Sigma) &= \|\Sigma\|^{-\frac{r}{2}}\|(R_{\mu}x)^{\top}\Sigma^{-1}(R_{\mu}x)\|^{-\frac{p}{2}}\\ \notag
    & = \|\Sigma\|^{-\frac{r}{2}}\|(x)^{\top}\Sigma^{-1}(x)\|^{-\frac{p}{2}}
\end{align}
where $R_{\mu} \in O(n)$ and covariance matrix $\Sigma$ is a symmetric positive-definite.

As clarified in section \ref{sec:grassmann}, we obtain Grassmann matrix by equivalently computing on $SPD_{m}^{++}$ with the diffeomorphic mapping $J^{-1}_{\mathcal{G}}=\mathtt{UU^T}$. 
 The error model is modified to be $F_{\Theta_{NN}} = Y^{\top}Y + N$, where both $Y^{\top}Y$ and $N$ are semi-positive definite matrix (SPD) $SPD^{++}_m$. The Gaussian distribution extended on SPD manifold, for a random $2th-$order tensor $X\in \mathbb{R}^{m\times m}$ (matrix) is
\begin{equation}
    P(X;M,S) = \frac{1}{\sqrt{(2\pi)^m}|\mathcal{S}|}\exp{\frac{1}{2}(X-M)\mathcal{S}^{-1}(X-M)}
\end{equation}
with mean $M\in \mathbb{R}^{m\times m}$ and $4th-$order
covariance tensor $\mathcal{S}\in \mathbb{R}^{m\times m\times m\times m}$ inheriting symmetries from three dimensions. To be more specific, $\mathcal{S}^{mn}_{ij} = \mathcal{S}^{nm}_{ij} = \mathcal{S}^{mn}_{ij}$ and $\mathcal{S}^{mn}_{ij} = \mathcal{S}^{ij}_{mn}$, thereby there is a vectorized version of Equation \ref{eq:grasspdf},
\begin{equation}
    P(X;\mu,\Sigma) = \frac{1}{\sqrt{(2\pi)^{\Tilde{m}}}|\Sigma|}\exp{\frac{1}{2}(X-\mu)\Sigma^{-1}(X-\mu)}
\end{equation}
where $\Tilde{m} = m + \frac{m(m-1)}{2}$, indicating the unfolded parameters:
\begin{figure*}
\begin{equation}
\label{eq:Smatrix}
\mu=
\begin{bmatrix}
\mu_{11}\\
\mu_{22}\\
\vdots\\
\mu_{mm}\\
\mu_{12}\\
\mu_{13}\\
\vdots\\
\mu_{(m-1)m}
\end{bmatrix}
\Sigma = 
  \begin{bmatrix}
  \mathcal{S}^{11}_{11}&\mathcal{S}^{22}_{11}&\ldots&\mathcal{S}^{mm}_{11}&\sqrt{2}\mathcal{S}^{12}_{11}&\sqrt{2}\mathcal{S}^{13}_{11}&\ldots&\sqrt{2}\mathcal{S}^{1m}_{11}\\
  \mathcal{S}^{11}_{22}&\mathcal{S}^{22}_{22}&\ldots&\mathcal{S}^{mm}_{22}&\sqrt{2}\mathcal{S}^{12}_{22}&\sqrt{2}\mathcal{S}^{13}_{22}&\ldots&\sqrt{2}\mathcal{S}^{1m}_{22}\\
  \vdots&\vdots& & \vdots & \vdots& & \vdots \\
  \mathcal{S}^{11}_{mm}&\mathcal{S}^{22}_{mm}&\ldots&\mathcal{S}^{mm}_{mm}&\sqrt{2}\mathcal{S}^{12}_{mm}&\sqrt{2}\mathcal{S}^{13}_{mm}&\ldots&\sqrt{2}\mathcal{S}^{1m}_{mm}\\
  
  \sqrt{2}\mathcal{S}^{12}_{11}&\sqrt{2}\mathcal{S}^{12}_{22}&\ldots&\sqrt{2}\mathcal{S}^{12}_{mm}&2\mathcal{S}^{12}_{12}&2\sqrt{2}\mathcal{S}^{13}_{12}&\ldots&2\mathcal{S}^{1m}_{12}\\
  \sqrt{2}\mathcal{S}^{13}_{11}&\mathcal{S}^{13}_{22}&\ldots&\sqrt{2}\mathcal{S}^{13}_{mm}&2\mathcal{S}^{13}_{13}&2\mathcal{S}^{13}_{13}&\ldots&2\mathcal{S}^{1m}_{13}\\
  \vdots&\vdots& & \vdots & \vdots& & \vdots \\
  \sqrt{2}\mathcal{S}^{1m}_{11}&\sqrt{2}\mathcal{S}_{22}^{2m}&\ldots&\sqrt{2}\mathcal{S}^{mm}_{1m}&2\mathcal{S}^{1m}_{12}&2\mathcal{S}^{13}_{1m}&\ldots&2\mathcal{S}^{mm}_{mm}\\
  \end{bmatrix}
\end{equation}
\end{figure*}

Since each entry of $X$ comes from rearranging DNN outputs, 
which can be assumed to not correlate with the other entries, 
there is $\mathcal{S}^{ab}_{cd}=1$ when $a=b=c=d$, 
otherwise $\mathcal{S}^{ab}_{cd}=0$,
then the likelihood could be further simplified with $\Sigma$ 
being an identity matrix. Then there is the proof of proposition \ref{prop:grass_MLE}.
\begin{proposition}
DEL with $J^{-1}_{\mathcal{G}}$  gives MLE of element on $\mathcal{G}(m.\mathbb{R}^n)$, if assuming $\Sigma$ is an identity matrix.
\end{proposition}
\begin{proof}
The log-likelihood function of PDF in Equation (\ref{eq:Smatrix}) is 
$\arg\max\limits_{Y\in SPD^{++}_m}L(Y;F_{\Theta_{NN}},\Sigma) = \arg\min\limits_{Y\in SPD^{++}_m}(\Tilde{F}_{\Theta_{NN}}-\Tilde{Y}))^{\top}(\Tilde{F}_{\Theta_{NN}}-\Tilde{Y}))=\arg\min\limits_{Y\in SPD^{++}_m}\|\Tilde{F}_{\Theta_{NN}}-\Tilde{Y}\|_F^2$, where $\Tilde{F}_{\Theta_{NN}}$ and $\Tilde{Y}$are  vectorized versions of $F_{\Theta_{NN}}$ and, $Y$ respectively. 
\end{proof}

\subsection{Generalization Ability}
\subsubsection{Proof of poposition \ref{prop:so_generalization}}
\begin{proposition}
\label{prop:so_generalization}
Any element of dimension $n$ on $SO(n)$ belongs to the image of $J^{-1}_{SO(n)}$ from known rotations within a certain range, 
if the Euclidean input of $J^{-1}_{SO(n)}$ is of more than $n$ dimensions.
\end{proposition}
\begin{proof}
The image of $J^{-1}_{SO(n)}$ is a set of $n$ orthogonal vectors which could span the $n$-dimensional vector space, or it could be regarded as parameterization of rotations around $n$ orthogonal axes, so they can be written as $\mathbf{\theta} = [\theta_1, \theta_2, \ldots, \theta_n]^{\top}\triangleq  \mathbf{u}\theta \triangleq \omega t \in \mathbb{R}^n $.

For the identity $\dot{R} = R[\omega]_{\times}\in T_RSO(n)$, with constant $\omega$, its solution $R(t) = R_0\exp([\omega]_{\times}t)\overset{R_0=I}{=}\exp([\omega]_{\times}t)=\exp{[\theta]_{\times}}=\sum\limits_k\frac{\theta^k}{k!}([u]_{\times})^k$, which means that if there are $n$ orthogonal basis, any rotation could be represented with transformations on $SO(n)$. 
\end{proof}

\newpage
\appendix
\onecolumn

\end{document}